\documentclass{article}

\usepackage[utf8]{inputenc} % allow utf-8 input
\usepackage[T1]{fontenc}    % use 8-bit T1 fonts
\usepackage{hyperref}       % hyperlinks
\usepackage{url}            % simple URL typesetting
\usepackage{booktabs}       % professional-quality tables
\usepackage{amsfonts}       % blackboard math symbols
\usepackage{nicefrac}       % compact symbols for 1/2, etc.
\usepackage{microtype}      % microtypography

\usepackage{amssymb,graphicx,bigints,xparse,xcolor}              % list packages between braces
\usepackage{algorithm,algorithmic,amsthm}
\usepackage{enumitem}
\usepackage{mathtools}
\usepackage{bigints}
\usepackage{relsize}
\usepackage{xcolor}
\usepackage{enumitem}
\usepackage{blkarray}
\usepackage{wrapfig}
\usepackage{lipsum}  % generates filler text

\usepackage[font=small,labelfont=bf]{caption}

\DeclareCaptionLabelFormat{andtable}{#1~#2  \&  \tablename~\thetable}

\newtheorem{theorem}{Theorem}

\newtheorem{lemma}[theorem]{Lemma}
\newtheorem{definition}{Definition}
\newtheorem{assumption}{Assumption}

\title{Randomized Policy Learning for Continuous State and Action MDPs}

\author{%
  Hiteshi Sharma, Rahul Jain and Shashank Hegde\\
  University of Southern California\\
  \texttt{(hiteshis,rahul.jain)@usc.edu} \\
}

\begin{document}

\maketitle

\begin{abstract}
Deep reinforcement learning methods
have achieved state-of-the-art results in a variety of challenging, high-dimensional
domains ranging from video games to locomotion. The key to success has been the use of deep neural networks  used to approximate the policy and value function. Yet, substantial tuning of weights is required for good results. We instead use randomized function approximation. Such networks are  not only cheaper than training fully connected networks but also improve the numerical performance.
We present \texttt{RANDPOL}, a generalized policy iteration algorithm for MDPs with continuous state and action spaces. Both the policy and value functions are represented with randomized networks.  
We also give finite time guarantees on the performance of the algorithm. Then we  show the numerical performance on challenging environments and compare them with deep neural network based algorithms.
  
\end{abstract}
\section{Introduction}

Recently, for continuous control tasks, reinforcement learning (RL) algorithms based on actor-critic architecture \cite{lillicrap2015continuous} or policy optimization \cite{schulman2017proximal} have shown remarkably good performance. The policy and the value function are represented by deep neural networks and then the weights are updated accordingly. However,  \cite{islam2017reproducibility} shows that the performance of these RL algorithms vary a lot with changes in hyperparameters, network architecture etc.  
Furthermore, \cite{mania2018simple} showed that a simple linear policy-based method with weights  updated by a  random search method can outperform some of these state-of-the-art results. A key question is how far we can go by relying almost exclusively on these architectural biases.

For Markov decision processes (MDPs) with discrete state and action spaces,  model-based algorithms based on dynamic programming (DP) ideas \cite{Put05} can be used when the model is known.  Unfortunately, in many problems (e.g., robotics), the system model is unknown, or simply too complicated to be succinctly stated and used in DP algorithms. Usually, latter is the more likely case. In such cases, model-free stochastic approximation algorithms like Q-learning are used which are known to be very slow to converge \cite{sutton2018reinforcement}.  For continuous spaces, these ideas are used in conjunction with function approximation for generalization. An alternative has been empirical algorithms such as \cite{munos2008finite,2017arXiv170907506H} which replace the expectation in the Bellman operator with a sample average approximation obtained by getting multiple samples of the next state for each state-action pair. This requires access to a generative model. At first glance, this may seem restrictive but in a large variety of applications, we have access to a physics  simulator engine. In fact, in robotics, for example, training of RL algorithms is first done offline in high-fidelity simulators as the training takes too long to converge for it to be done directly on the physical robot.

%Dynamic programming has been studied extensively in the RL setting \cite{Put05,sutton2018reinforcement}. 

We focus on (approximate) dynamic programming algorithms based on actor-critic architecture for RL with access to a generative model (a simulator). 
We consider both the state space and the action space to be continuous (together referred to as a Continuous MDP).  So, we need function approximation for both policy and value function. In this paper, we consider \emph{randomized} networks where the connections in bottom layer(s) are left untrained after initialization and only the last layer is finely tuned.  Such networks have been studied extensively, both for shallow \cite{rahimi2008uniform,rahimi2009weighted} and deep architecture \cite{2020arXiv200212287G}. Thus, the training algorithm only needs to operate on a reduced set of weights with similar or even better performance with respect to fully trained architectures. This is different from using the last-hidden layer of deep neural networks as a feature extractor and  updating the last layer with different algorithm \cite{levine2017shallow}, which still trains a fully connected network. Furthermore, \cite{rahimi2008uniform} shows that such random bases correspond to a reproducing kernel Hilbert space (RKHS), which are known to be dense in the space of continuous functions. They also provide theoretical bounds on the error due to approximation by finite random features. 

The main contributions of this paper are: First, an algorithm that is easy  to interpret,  and can be viewed as a generalized policy iteration algorithm in combination with  randomized function approximation. Secondly, we give finite-time theoretical guarantees unlike most of the RL algorithms for continuous MDPs with only asymptotic convergence analysis, or none at all. Third, and most important of all, we demonstrate that the algorithm works, and works better than state-of-the-art algorithms like PPO and DDPG, on a quadrupedal robot problem in the Minitaur environment.  To the best of our  knowledge,  no previous work on continuous MDPs have provided algorithms that work empirically on complicated control problems, and also provided a  theoretical  analysis.  

The rest of the paper is organized as follows. Section \ref{sec:prob_form} presents the problem formulation. Section \ref{sec:algo} presents the algorithm and the convergence result followed by the theoretical analysis in the next section. The last section focuses on numerical experiments.

\section{Preliminaries}
\label{sec:prob_form}
Consider an MDP $(\mathcal{X}, \mathcal{U}, P,r,\gamma)$ where $\mathcal{X}$ is the state space and $\mathcal{U}$ is the action space. The transition probability  kernel is given by $P(\cdot | x,u)$, i.e., if action $u$ is executed in state $x$, the probability that the next state is in a Borel-measurable set $B$ is $P(x_{t+1} \in B | x_t = x, u_t =u)$ where $x_t$ and $u_t$ are the state and action at time $t$. The reward function is $r: \mathcal{X} \times \mathcal{U} \rightarrow \mathbb{R}$. We are interested in maximizing the long-run expected discounted reward where the discount parameter is $\gamma$.

Let $\Pi$ denote the class of stationary deterministic Markov policies mappings $\pi: \mathcal{X} \to \mathcal{U}$ which only depend on history through  the current state. We only consider such policies since it is well known that there
is an optimal MDP policy in this class. When the initial state is given, any policy $\pi$ determines a probability measure $P^{\pi}$. Let the expectation with respect to this measure be $\mathbb{E}^{\pi}$. We focus on infinite horizon discounted reward criterion. The expected discounted reward or the action-value function for a policy $\pi$ and initial state $x$ and action $u$ is given as
\[
Q^{\pi}(x,u) := \mathbb{E}^{\pi} \left[\sum_{t=0}^{\infty} \gamma^t \, r(x_t,u_t) \bigg| x_0 = x, u_0 =u \right]
\] 
The optimal value function is given as
\[
Q^{*}(x) := \sup_{\pi \in \Pi} \mathbb{E}^{\pi} \left[\sum_{t=0}^{\infty} \gamma^t \, r(x_t,u_t) \bigg| x_0 = x,u_0 = u \right]
\] 
and the policy which maximizes the value function is the optimal policy, $\pi^*$. Now we make the following assumptions on the regularity of the MDP.
\begin{assumption}(\textbf{Regularity of MDP})
The state space $\mathcal{X}$ and the action space $\mathcal{U}$ are compact subset of $d_X$ and $d_U$ dimensional Euclidean spaces respectively. The rewards are uniformly bounded by $r_{\max}$, i.e., $r(x,u) \leq r_{\max}$, for all $(x,u) \in \mathcal{X} \times \mathcal{U}$. Furthermore, $\mathcal{U}$ is convex.
\label{assum:reg_mdp}
\end{assumption}
The assumption above implies that for any policy $\pi$, $Q^{\pi} \leq Q_{\max} = r_{\max}/(1-\gamma)$. The next assumption is on Lipschitz continuity of MDP in action variable.

\begin{assumption}(\textbf{Lipschitz continuity})
The reward and the transition kernel are Lipschitz continuous with respect to the action i.e., there exist constants $L_r$ and $ L_p $ such that for all $(x,u,u') \in \mathcal{X} \times \mathcal{U} \times \mathcal{U}$ and a measurable set $B$ of $\mathcal{X}$, the following holds:
\begin{align*}
\left|r(x,u) - r(x,u') \right|&\leq L_r \|u-u'\| ~~\text{and }~~
\left |P(B|x,u) - P(B|x,u') \right| &\leq L_p \|u-u'\|.
\end{align*}
\label{assum:lips_mdp}
\end{assumption}
The compactness of action space combined with Lipschitz continuity implies that the greedy policies do exist. Let $\mathcal{B}(\mathcal{X},\mathcal{U})$ be the set of  functions on $\mathcal{X}$ and $\mathcal{U}$ such that $\|f\|_{\infty} \leq Q_{\max}$.  
Let us now define the Bellman operator for action-value functions $G: \mathcal{B}(\mathcal{X}, \mathcal{U}) \to \mathcal{B}(\mathcal{X},\mathcal{U})$ as follows
\[
G\, Q(x,u) :=  r(x,u) + \gamma \, \mathbb{E}_{x' \sim P(\cdot|x,u)} \max_{u'} Q(x',u').
\]
It is well known that the operator $G$ is a contraction with respect to $\| \cdot\|_{\infty}$ norm and the contraction parameter is the discount factor $\gamma$. Hence, the sequence of iterates $Q_k = G\, Q_{k-1}$ converge to $Q^*$ geometrically. Since we will be using the $L_2$ norm, we do not have a contraction property with respect to it. Hence, we need bounded  Radon-Nikodym derivatives of transition probabilities which we illustrate in the next assumption. Such an assumption has been used earlier with finite action spaces \cite{munos2008finite,haskell2017randomized} and for continuous action spaces in \cite{antos2008fitted}.
\begin{assumption}(\textbf{Stochastic Transitions})
	\label{assum:radon}  For all $\left(x,\,u\right)\in\mathcal{X}\times\mathcal{U}$,
	$P\left(\cdot\,\vert\,x,\,u\right)$ is absolutely continuous with
	respect to $\mu$ and $C_{\mu}\triangleq\sup_{\left(x,\,u\right)\in\mathcal{X}\times\mathcal{U}}\left\Vert \frac{dP\left(\cdot\,\vert\,x,\,u\right)}{d\mu}\right\Vert _{\infty}<\infty$.
\end{assumption}

Since we have a sampling based algorithm, we need a function space to approximate value functions. In this paper, we focus on randomized function approximation via random features. 
Let $\Theta$ be a set of parameters and let $\phi\text{ : } \mathcal{X} \times \mathcal{U}\times\Theta\rightarrow\mathbb{R}$. The feature functions need to satisfy $\sup_{\left(x,u,\,\theta\right)\in\mathcal{X} \times \mathcal{U}\times\Theta}|\phi\left(x,u;\,\theta\right)|\leq1$, e.g., Fourier features. We define
\[
\mathcal{F}\left(\Theta\right) := \left\{ f: f\left(x,u\right)=\int_{\Theta}\phi\left(x,u;\,\theta\right)\alpha\left(\theta\right)d\theta,\, \,|\alpha\left(\theta\right)|\leq C\,\nu\left(\theta\right),\,\forall\theta\in\Theta\right\} .
\]
We are interested in finding the best fit within
finite sums of the form $\sum_{j=1}^{J_Q}\alpha_{j}\phi\left(x,u;\,\theta_{j}\right)$.
Doing classical function fitting with $\sum_{j=1}^{J_Q}\alpha_{j}\phi\left(x,u;\,\theta_{j}\right)$
leads to nonconvex optimization problems because of the joint dependence
in $\alpha$ and $\theta$. Instead, we fix a density $\nu$ on $\Theta$
and draw a random sample $ \theta_{j}$
from $\Theta$ for $j=1,2,\ldots J_Q$. Once these $\left( \theta_{j}\right) _{j=1}^{J_Q}$
are fixed, we consider the space of functions,
\[
\widehat{\mathcal{F}}\left(\theta^{1:J_Q}\right):=
\left\{ f\left(x,u\right)=\sum_{j=1}^{J_Q}\alpha_{j}\phi\left(x,u;\,\theta_{j}\right)\,\vert\,\|\left(\alpha_{1},\ldots,\,\alpha_{J_Q}\right)\|_{\infty}\leq C/J_Q\right\} .
\]
Now, it remains to calculate weights $\alpha$ by minimizing a convex loss. We also need a function space for approximating the policy. Since $\mathcal{U}$ is multi-dimensional, policy $\pi:\mathcal{X} \to \mathcal{U}$ is $\pi = (\pi_1, \pi_2, \ldots \pi_{d_U})$. For each co-ordinate $k$, we  define $\Pi_k(\Theta)$, a function space similar to $\mathcal{F}(\Theta)$ but with functions defined just over the state space.
Let $\psi\text{ : }\mathcal{X}\times\Theta\rightarrow\mathbb{R}$
be a feature function, then for $1\leq k \leq d_U$, define the co-ordinate projection space,
\[
\Pi_k(\Theta) := \left\{ f: f\left(x\right)=\int_{\Theta}\psi\left(x;\,\theta\right)\alpha\left(\theta\right)d\theta,\,\,|\alpha\left(\theta\right)|\leq C'\,\nu\left(\theta\right),\,\forall\theta\in\Theta\right\} .
\]
Let  $\widehat{\Pi}_k\left(\theta^{1:J_{\pi_k}}\right)$ denote an approximation of $\Pi_k(\Theta)$ which is defined similar to $\widehat{\mathcal{F}}\left(\theta^{1:J_{Q}}\right)$. 
The rationale behind choosing randomized function spaces (where the parameters are chosen randomly) is that \emph{randomization is cheaper than optimization}. They can be thought of networks where the bottom layers are randomly fixed and only the last layer is finely tuned. This not only saves the number of trainable parameters but also shows good empirical performance.

Furthermore, let us define the $L_{1,\mu}$ norm of a function for a given a probability distribution $\mu$ on $\mathcal{X} \times \mathcal{U}$   as  $\|f\|_{1,\,\mu} := \left(\int_{\mathcal{X} \times \mathcal{U}}|f\left(x\right)| d \mu \right)$. The empirical norm at given samples $((x_1,u_1), (x_2,u_2), \ldots (x_N,u_N))$ is defined as $\|f\|_{1,\,\hat{\mu}} := \frac{1}{N} \sum_{i=1}^N |f(x_i,u_i)|$.

\section{The Algorithm}
\label{sec:algo}
 We now present our RANDomized POlicy Learning (\texttt{RANDPOL}) algorithm. It  approximates both action-value function and policy, similar to actor-critic methods. It comprises of two main steps: policy evaluation and policy improvement. Given a policy $\pi: \mathcal{X} \to \mathcal{U}$, we can define a policy evaluation operator $G^{\pi}: B(\mathcal{X}, \mathcal{U}) \to B(\mathcal{X},\mathcal{U})$ as follows
 \[
G^{\pi}\, Q(x,u) =  r(x,u) + \gamma \, \mathbb{E}_{x' \sim P(\cdot|x,u)}  Q(x',\pi(x')) .
\]
 Note that if $\pi$ is a greedy policy with respect to  $Q$, i.e., $\pi(x) \in \arg\max_{u \in \mathcal{U} }Q(x,u)$, then indeed $G^{\pi}\, Q = G\, Q$. When there is an uncertainty in the underlying environment, computing expectation in the Bellman operator is expensive. If we have a generative model of the environment, we can replace the expectation by an empirical mean leading to definition of an empirical Bellman operator for policy evaluation for a given policy $\pi$:
\begin{equation}
\widehat{G}^{\pi}_M\, Q(x,u) := \left[r(x,u) + \frac{\gamma}{M}\sum_{i=1}^M   Q(x_i^{x,u},\pi(x_n)) \right].
\label{eq:empBell}
\end{equation}
where $x_i^{x,u} \sim P(\cdot|x,u)$ for $i=1,2,\ldots, M$. Note that the next state samples, $x'$,  are i.i.d. If the environment is deterministic, like Atari games or locomotion tasks, having a single next state suffices and we don't need a generative model.  

For each iteration, we first sample $N_Q$ state-action pairs $\{(x_1,u_1), (x_2,u_2),\ldots (x_{N_Q},u_{N_Q})\}$ independently from $\mathcal{X} \times \mathcal{U}$. Then, for each sample, we compute $\widehat{Q}_M(x_n,u_n) = \widehat{G}^{\pi}_M\, Q(x_n,u_n)$ for given action-value function $Q$ and policy $\pi$. 
Given the data $\left\{ \left((x_{n},u_{n}),\,\widehat{Q}\left(x_{n},u_n\right)\right)\right\} _{n=1}^{N_Q}$,
we fit the value function over the state and action space by computing a
best fit within $\widehat{\mathcal{F}}\left(\theta^{1:J}\right)$
by solving
\begin{align}
\min_{\alpha}\, & \frac{1}{N_Q}\sum_{n=1}^{N_Q}|\sum_{j=1}^{J_Q}\alpha_{j}\phi\left(x_{n},u_n;\,\theta_{j}\right)-\widehat{Q}_M\left(x_{n},u_n\right)|^{2} \label{eq:func_approx_Q}\\ 
\text{s.t.}\, ~~& \|\left(\alpha_{1},\ldots,\,\alpha_{J}\right)\|_{\infty}\leq C/J_Q. \nonumber
\end{align}
This optimization problem only optimizes
over weights $\alpha^{1:J}$ since parameters $\theta^{1:J}$ have
already been randomly sampled from a given distribution $\nu$. This completes the policy evaluation step.

Next, the algorithm does the policy improvement step. For a fixed value function $Q \in B(\mathcal{X}, \mathcal{U})$, define 
 $\widetilde{\pi}(x) \in \arg\max_{u \in \mathcal{U}} Q(x,u).$
  If action space were discrete and we had good approximation of value function, we could have just followed the greedy policy. But for our setting, we will need to approximate the greedy policy too. 
Let us compute a greedy policy empirically given $N_{\pi}$ independent samples $\{x_1,x_2, \ldots x_{N_{\pi}} \}$ and value function $Q \in B(\mathcal{X}, \mathcal{U})$, as follows:
\begin{equation}
  {\pi} (x) \in \arg \max_{\widehat{\pi} \in\widehat{\Pi}\left(\theta^{1:J_{\pi}}\right)} \frac{1}{N_{\pi}} \sum_{i=1}^{N_{\pi}}Q(x_i,\widehat{\pi}(x_i))
   \label{eq:empPoImp}
\end{equation}
 where $\widehat{\Pi} = \left \{\widehat{\pi}: \widehat{\pi} = \left(\widehat{\pi}_1, \ldots \widehat{\pi}_{d_U} \right),  \widehat{\pi}_k \in \widehat{\Pi}_k\left(\theta^{1:J_{\pi_k}}\right) \right\}$ and $J_{\pi} = \sum_{k=1}^{d_u}  J_{\pi_k}$. With this empirical policy improvement step, we now present the complete  \texttt{RANDPOL} algorithm, shown in Algorithm \ref{alg:EPL}, where 
  we initialize the algorithm with a random value function $Q_0$ and $\pi_0$ is the approximate greedy policy with respect to $Q_0$ computed by equation (\ref{eq:empPoImp}). 
\begin{algorithm}[!tph]
	\caption{\label{alg:EPL} \texttt{RANDomized POlicy Learning}}
	
	Input:  sample sizes
	$N_Q,M,J_Q,N_{\pi},J_{\pi}$;  initial value function $Q_{0}$\\
	
 For $k=0, 1, 2, \ldots $,
	\begin{enumerate}
		\item Sample $\{x_{n},u_n\} _{n=1}^{N_Q}$ from state and action space
		\item Empirical policy evaluation: 
		\begin{enumerate}
		\item For each sample $(x_n,u_n)$, compute $\widehat{Q}_M(x_n,u_n) = \widehat{G}^{\pi_k}_M\, Q_k(x_n,u_n)$
		\item Approximate $Q_{k+1}$ according to (\ref{eq:func_approx_Q})
		\end{enumerate}
		\item Sample $\{x_{n}\} _{n=1}^{N_{\pi}}$ from state space
		\item Empirical policy improvement: 
		\begin{enumerate}
		\item Approximate $\pi_{k+1}$ according to (\ref{eq:empPoImp})
		\end{enumerate}
	\end{enumerate}
	
\end{algorithm}

Define the policy improvement operator $H:\mathcal{B}(\mathcal{X}, \mathcal{U}) \to \mathcal{B}(\mathcal{X})$ as $H\, Q (x) := \sup_{u \in \mathcal{U}} Q(x,u).$
If policy $\pi$ was fixed, then $H^{\pi}\, Q (x) =  Q(x,\pi(x))$
 will give the performance of the policy. 
To measure the function approximation error, we next define distance measures for function spaces:
\begin{itemize}
	\item $d_{1}\left(\pi, \,\mathcal{F}\right):= \sup_{f \in\mathcal{F}} \inf_{f'\in\mathcal{F}}\|f'-G^{\pi}\,f\|_{1,\,\mu}$
	is the approximation error for a specific policy $\pi$;
	\item $d_{1}\left(\Pi,\,\mathcal{F}\right):=  \sup_{\pi \in \Pi} d_{1,\,\mu}\left(\pi, \,\mathcal{F}\right)$
	is the inherent Bellman error for the entire class $\Pi$;  and
	\item $e_{1}\left(\Pi,\,\mathcal{F}\right) := \sup_{Q \in\mathcal{F}} \inf_{\pi \in {\Pi}}\|H\, Q-H^{\pi}\,Q\|_{1,\,\mu}$ is the worst-case approximation error of the greedy policy.
\end{itemize}

Let us define:
\begin{align*}
&J_Q^0(\epsilon, \delta) :=\left[\frac{5\,C}{\epsilon}\left(1+\sqrt{2\,\log\frac{5}{\delta}}\right)\right]^{2},~~~
J_{\pi}^0(\epsilon, \delta) :=\left[\frac{3\,L_U\,C'}{\epsilon}\left(1+\sqrt{2\,\log\frac{3}{\delta}}\right)\right]^{2},\\
&M^0(\epsilon, \delta):=\left(\frac{2\,v_{\max}^{2}}{(\epsilon/5)^{2}}\right)\log\left[\frac{10\,N}{\delta}\right],\\
&N_Q^0(\epsilon, \delta) := \left(\frac{128\,v_{\max}^{2}}{\left(\epsilon/5\right)^{2}}\right)\,\log\left[\frac{40\,e\left(J_Q+1\right)}{\delta}\left(\frac{2\,e\,v_{\max}}{(\epsilon/5)}\right)^{J_Q}\right],\\
%L^0(\epsilon, \delta) &= \left(\cfrac{(L_r + \gamma\, v_{\max}\,L_p) \, \text{diam}(\mathcal{U})}{\epsilon/5} \right)^{d_U}\log \cfrac{5\,N}{\delta},\\
 \text{and }\quad  &N_{\pi}^0(\epsilon, \delta) := \left(\frac{128\,v_{\max}^{2}}{\left(\epsilon/3\right)^{2}}\right)\,\log\left[\frac{24\,e\left(J_{\pi}+1\right)}{\delta}\left(\frac{2\,e\,v_{\max}}{(\epsilon/3)^2}\right)^{J_{\pi}}\right].
\end{align*}

\begin{theorem}
Let Assumptions \ref{assum:reg_mdp}, \ref{assum:lips_mdp} and \ref{assum:radon} hold. Choose an $\epsilon>0$ and $\delta \in (0,1)$. Set $\delta' = 1-(1/2 + \delta/2)^{1/(K^*-1)}$ and denote $\widetilde{C}:=4\left(\frac{1-\gamma^{K+1}}{(1-\gamma)^2}\right) C_{\mu}$ and
\begin{equation}
\label{eq:k_star}
K^* := \left  \lceil \cfrac{\log\left( C_{\mu}\,\epsilon\right) - \log \left(2\, Q_{\max}\right)}{\log \gamma} \right\rceil .
\end{equation}

Then, if $N_{Q} \geq {N_Q^0}(\epsilon,\delta')$, $N_{\pi} \geq {N_{\pi}^0}(\epsilon,\delta')$,  $M \geq M^0(\epsilon,\delta')$, $L \geq L^0(\epsilon,\delta')$, $J_Q \geq  J_Q^0(\epsilon,\delta')$, $J_{\pi} \geq  J_{\pi}^0(\epsilon,\delta')$  and $K\geq\log\left(4/\left(\left(1/2-\delta/2\right)\left(1-q\right)q^{K^{*}-1}\right)\right),$
	we have that 	with probability at least $1- \delta$,
	\begin{align*}
	 \|Q^{\pi_K}-Q^{*}\|_{1,\,\mu}\leq \widetilde{C}\,\epsilon.
	\end{align*}
\label{thm:main}
\end{theorem}

\paragraph{Remark:} The above theorem states that if we have sufficiently large number of  samples, in particular, $J_{\pi}, J_{Q}$ and $ M = O\left(1/{\epsilon^2} \log1/\delta\right)$ and $N_Q, N_{\pi} = O\left(1/{\epsilon^2} \log1/\epsilon^{J}\delta\right) $ then for sufficiently large iterations, the approximation error can be made arbitrarily small with high probability. Moreover, if  Lipschitz continuity assumption is not satisfied then the result can be presented in a more general form:
\[
	 \|Q^{\pi_K}-Q^{*}\|_{1,\,\mu}\leq 2\left(\frac{1-\gamma^{K+1}}{(1-\gamma)^2}\right)
C_{\mu} 	\left[d_{1}\left(\Pi\left(\Theta\right),\,\mathcal{F}\left(\Theta\right)\right) +  \gamma \, C_{\mu} e_{1}\left(\Pi\left(\Theta\right),\,\mathcal{F}\left(\Theta\right)\right)+ \epsilon\right].	
\]

\section{Theoretical Analysis}
In this section, we will analyze Algorithm \ref{alg:EPL}. First, we will bound the error in one iteration of \texttt{RANDPOL} and then analyze how the errors propagate through iterations. 

\subsection{Error in one iteration}
Since \texttt{RANDPOL} approximates at two levels: policy evaluation and policy improvement, we decompose the error in one iteration as the sum of approximations at both levels. If a function $Q$ was given as an input to \texttt{RANDPOL} for an iteration, the resulting value function, $Q'$,  can be written as an application of a random operator $\widehat{G}$. This random operator depends on the input sample sizes $N_Q,M,J_Q,N_{\pi}, J_{\pi}$ and the input value function $Q$. Let $\pi$ be the approximate greedy policy with respect  to $Q$. Thus, we have $Q' = \widehat{G}(N_Q,M,J_Q,N_{\pi},J_{\pi})\, Q$. For concise notation, we will just write $Q' = \widehat{G} \, Q$. Let us now decompose this error into policy evaluation and policy improvement approximation errors.
\begin{align*}
Q' &=  G\, Q + (Q' - G^{\pi}\, Q) + (G^{\pi}\, Q- G\, Q)\\
&= G\, Q + \epsilon' + \epsilon'' = G\, Q + \epsilon,
\end{align*}
where $\epsilon' = Q' - G^{\pi}\, Q$,  $\epsilon'' = G^{\pi}\, Q- G\, Q$ and $ \epsilon =  \epsilon' +  \epsilon''$. In other words, $ \epsilon'$ is the approximation error in policy evaluation and $ \epsilon''$ is the error in policy improvement. If we get a handle on both these errors, we can bound $ \epsilon$ which is the error in one step of our algorithm.

\paragraph*{Policy evaluation approximation} 
Let us first bound the approximation error in policy evaluation, $\epsilon' $. The proof is given in the supplementary material. 
\begin{lemma}
	Choose $Q\in\mathcal{F}\left(\Theta\right)$, $\epsilon>0$, and $\delta \in\left(0,\,1\right)$. 	 Also choose $N_Q \geq N_Q^0(\epsilon, \delta), M \geq M^0(\epsilon, \delta)$ and $J_Q \geq J_Q^0(\epsilon, \delta) $. Then, for
	$Q' = \widehat{G}\, Q$, the output of one iteration of our algorithm, we have $$\|Q' - {G}^{\pi}
	\,Q\|_{1,\,\mu}\leq d_{1}\left(\Pi\left(\Theta\right),\,\mathcal{F}\left(\Theta\right)\right)+\epsilon$$
	with probability at least $1-\delta$.
	\label{lem:one_step_poleval} 
\end{lemma}
\paragraph*{Policy improvement approximation} The second step in the algorithm is policy improvement and we now bound the approximation error in this step, $\epsilon''$. The proof is given in the supplementary material.

\begin{lemma}
	Choose $Q\in\mathcal{F}\left(\Theta\right)$, $\epsilon>0$, and $\delta \in\left(0,\,1\right)$. Let the greedy policy with respect to $Q$ be $\widetilde{\pi}(x) \in \arg \max_{u} Q(x,u)$.	 Also choose $N_{\pi} \geq N_{\pi}^0(\epsilon, \delta)$  and $J_{\pi} \geq J_{\pi}^0(\epsilon, \delta) $. Then, if the policy ${\pi}$ is computed with respect to $Q$ in equation (\ref{eq:empPoImp}), the policy improvement step in our algorithm, we have 
	$$\|G^{\pi}\,Q - G\,Q \|_{1,\mu}\leq \gamma \, C_{\mu} \left[e_{1}\left(\Pi\left(\Theta\right),\,\mathcal{F}\left(\Theta\right)\right)+\epsilon\right]$$
	with probability at least $1-\delta$.
	\label{lem:one_step_polimp} 
\end{lemma}
\subsection{Stochastic dominance.}
After bounding the error in one step, we will now bound the error when the random operator $\widehat{G}$ is applied iteratively by constructing a dominating Markov chain. 
Since we do not have a contraction with respect to the $L_1$ norm, we need an upper bound on how the errors propagate with iterations. 
Recall that $\widehat{G}\, Q_k = G \,Q_k + \epsilon_k$,  we use the point-wise error bounds as computed in the previous section. For a given error tolerance, it gives a bound on the number of iterations which we call $K^*$ as shown in equation (\ref{eq:k_star}). The details of the choice of $K^*$ is given in the appendix.
 We then construct a stochastic process as follows. We call iteration $k$ ``good'',  if the error $\|\epsilon_{k}\|_{1,\,\mu} \leq \epsilon$ and ``bad'' otherwise.  We then construct a stochastic process $\left\{ X_{k}\right\} _{k\geq0}$ with state space $\mathcal{K}$ as  $:=\left\{ 1,\,\,2,\ldots,\,K^{*}\right\} $ such that
\[
X_{k+1}=\begin{cases}
\max\left\{ X_{k}-1,\,1\right\} , & \text{if iteration \ensuremath{k} is "good"},\\
K^{*}, & \text{otherwise}.
\end{cases}
\]
The stochastic process $\left\{ X_{k}\right\} _{k\geq0}$ is easier
to analyze than $\left\{ v_{k}\right\} _{k\geq0}$ because it is defined
on a finite state space, however $\left\{ X_{k}\right\} _{k\geq0}$
is not necessarily a Markov chain. Whenever $X_{k}=1$, it means that
we just had a string of $K^{*}$ ``good'' iterations in a row, and
that $\|Q^{\pi_k}-Q^{*}\|_{1,\,\mu}$ is as small as desired. 

We next construct a ``dominating"  Markov chain $\left\{ Y_{k}\right\} _{k\geq0}$
to help us analyze the behavior of $\left\{ X_{k}\right\} _{k\geq0}$.
We construct $\left\{ Y_{k}\right\} _{k\geq0}$ and
we let $\mathcal{Q}$ denote the probability measure of $\left\{ Y_{k}\right\} _{k\geq0}$. 
Since $\left\{ Y_{k}\right\} _{k\geq0}$
will be a Markov chain by construction, the probability measure $\mathcal{Q}$
is completely determined by an initial distribution on $\mathbb{R}$
and a transition kernel for $\left\{ Y_{k}\right\} _{k\geq0}$. We now use the bound on one step error as presented in previous section which states that when the samples are sufficiently large enough for all $k $,
$$\mathbb{P}\left(\|\epsilon_k\|_{1,\mu} \leq \epsilon \right) > q(N_Q,M,J_Q,N_{\pi},J_{\pi}).$$
Denote this probability  by $q$ for a compact notation. Initialize $Y_{0}=K^{*}$, and  construct the process
\[
Y_{k+1}=\begin{cases}
\max\left\{ Y_{k}-1,\,1\right\} , & \mbox{w.p. }q,\\
K^{*}, & \mbox{w.p. }1-q,
\end{cases}
\]
where $q$ is the probability of a ``good'' iteration which increases with sample sizes $N,M,J$ and $L$.
We now describe a stochastic dominance relationship between the two
stochastic processes $\left\{ X_{k}\right\} _{k\geq0}$ and $\left\{ Y_{k}\right\} _{k\geq0}$.
We will establish that $\left\{ Y_{k}\right\} _{k\geq0}$ is ``larger''
than $\left\{ X_{k}\right\} _{k\geq0}$ in a stochastic sense. This
relationship is the key to our analysis of $\left\{ X_{k}\right\} _{k\geq0}$.
\begin{definition}
	Let $X$ and $Y$ be two real-valued random variables, then $X$ is
	\textit{stochastically dominated} by $Y$,  $X\leq_{st}Y$,
	 when $\mbox{Pr}\left\{ X\geq\theta\right\} \leq\mbox{Pr}\left\{ Y\geq\theta\right\} $
	for all $\theta$ in the support$(Y)$.
\end{definition}
The next lemma uses stochastic dominance to show that if the error in each iteration is small, then after sufficient iterations we will have small approximation error. The proof is given in the supplementary material.  
\begin{lemma}
	\label{lem:errprop} Choose $\epsilon>0$, and $\delta\in\left(0,\,1\right)$,
	and suppose $N,M,J$ and $L$ are chosen sufficiently large enough such that $\mathbb{P}\left(\|\epsilon_k\|_{1,\mu} \leq \epsilon \right) > q$ for all $k \geq 0$.
	Then, for $q\geq\left(1/2+\delta/2\right)^{1/\left(K^{*}-1\right)}$
	and $K\geq\log\left(4/\left(\left(1/2-\delta/2\right)\left(1-q\right)q^{K^{*}-1}\right)\right),$
	we have 	with probability at least $1-\delta$,
	\begin{align*}
	 \|Q^{\pi_K}-Q^{*}\|_{1,\,\mu}\leq 2\left(\frac{1-\gamma^{K+1}}{1-\gamma}\right)^{1/2}
	\left[C_{\mu}^{1/2}\epsilon+\gamma^{K/2}\left(2\,Q_{\max}\right)\right].
	\end{align*}

\end{lemma}
Now to prove Theorem \ref{thm:main}, we combine Lemmas \ref{lem:one_step_poleval} and  \ref{lem:one_step_polimp} to bound the error in one iteration. Next, we use Lemma \ref{lem:errprop} to bound the error after sufficient number of iterations. Also, since our function class forms a reproducing kernel Hilbert space \cite{rahimi2008uniform}, which is dense in the space of continuous functions, the function approximation error is zero.

\section{Numerical Experiments}
In this section, we present experiments showcasing the improved performance attained by our proposed algorithm compared to state-of-the-art deep RL methods. In the first part, we try it on a simpler environment, where one can compute the optimal policy theoretically. The second part focuses on a challenging, high-dimensional environment of  a quadrupedal robot.

\subsection{Proof of Concept}
We first test our proposed algorithm on a synthetic example where we can calculate the optimal value function and optimal policy analytically. In this example, $\mathcal{X} = [0,1]$
and $\mathcal{U} = [0,1]$. The reward is $r(x,u) = -(x-u)^2$ and $P(y|x,u) =\text{Unif}[u,1]$. 
The optimality equation for action-value function can be written as:
\[
Q(x,u) = -(x-u)^2 + \cfrac{\gamma}{1-u}\int_{u}^1 \, \max_w\, Q(y,w) dy.
\]
The value function is $v(x) = \max_u Q(x,u)$. For this example , $v^*(x) = 0$ and $\pi^*(x) = x$.
We ran the experiment with $N_q = 100, N_{\pi} = 100, J_q =20, J_{\pi} = 20$ and discount factor $\gamma = 0.7$. Fig. \ref{fig:error} shows the optimal policy and error in the performance of the policy $\|v^{\pi} - v^*\|_{\infty}$. The approximate policy is computed for $M=10$ and is very close to the optimal policy. The figure with performance error shows that even with a  small number of next state samples in the empirical policy evaluation step in the algorithm, we are able to achieve good performance.
\begin{figure}
	\includegraphics[width=2.5in, height =1.6in]{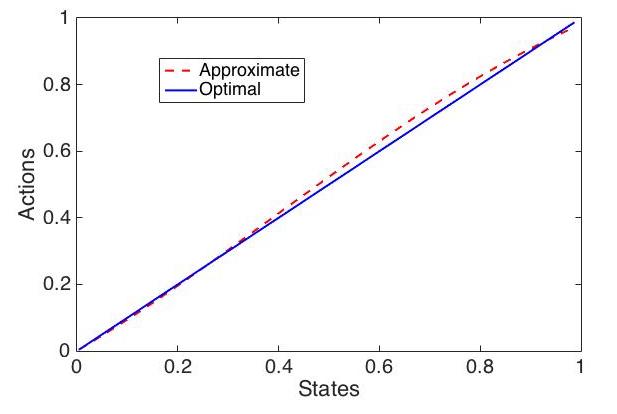}
	\includegraphics[width=2.5in, height =1.6in]{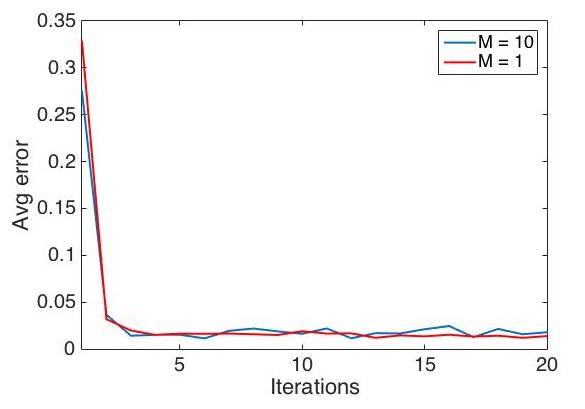}
	\caption{\textbf{Performance of \texttt{RANDPOL} on a synthetic example.} \textbf{Left:} Optimal and approximate policy.  \textbf{Right:} Error with iterations }
	\label{fig:error}
\end{figure}

\subsubsection{Lunar Lander}
In this problem, we used OpenAI gym Lunar Lander environment where the goal is to land a space-ship smoothly in a landing pad (marked by two flags). The state of this environment is represented as $8$-dimensional vector which has position, velocity, lander angle, angular velocity and contact points of the space-ship. The action is $2$-dimensional vector with each value from $-1$ to $1$ which controls the engine power and the direction. The reward depends on the distance from the landing pad. We train our landing agent using our algorithm and compare it with DDPG and PPO.  DDPG is also an actor-critic method but needs a fully connected neural network to be trained while we use randomized neural networks.  Table \ref{tab:2} shows the performance (averaged over $5$ random seeds) of DDPG, PPO and  \texttt{RANDPOL}  along with the number of parameters updated by the algorithm at each step. We see that with significant reduction in trainable parameters, we still have the better performance  for \texttt{RANDPOL}. This clearly shows that we can reduce the number of trainable parameters in the network without compromising on the performance. 

\begin{table}
\centering
 \caption{Maximal reward after 3M steps for Lunar Lander. Reward is the average reward over $5$ random seeds}\label{tab:2}
 \begin{tabular}{||c c c||} 
 \hline
Algorithm &Trainable parameters & Reward\\ [0.5ex] 
 \hline\hline
\texttt{RANDPOL} & 300 & -103.17  \\ 
 \hline
DDPG & 68600 & -195.45\\
 \hline
PPO & 68400  & -153.17 \\
 \hline

  \end{tabular}
  \label{tab:score}
 \end{table}

\subsection{Minitaur}
In this example, we focus on forward locomotion of a quadrupedal robot. The state is a $28$ dimensional vector consisting of position, roll, pitch, velocity etc. The action space is $8$ dimensional vector of torques on the legs. The physics engine for this environment is given by PyBullet \footnote{\texttt{https://github.com/bulletphysics/bullet3/blob/004dcc34041d1e5a5d92f747296b0986922ebb96/examples/\\pybullet/gym/pybullet\_envs/minitaur/envs/minitaur\_gym\_env.py}}. The reward has four components: reward for moving forward, penalty for sideways translation, sideways rotation, and energy expenditure. In our experiments, we maintain a experience replay buffer, previously used in \cite{mnih2013playing,lillicrap2015continuous}, which stores the data from past policies. We sample the data from this buffer which also helps in breaking the correlation among them. For exploration, we use an Ornstein-Uhlenbeck process \cite{lillicrap2015continuous}. We compare against the popular deep RL algorithms: DDPG \cite{lillicrap2015continuous} and PPO \cite{schulman2017proximal}. 
\begin{wraptable}{r}{5.5cm}
%\begin{table}
%\centering
 \caption{Maximal reward after 3M steps}\label{wrap-tab:1}
 \begin{tabular}{||c c||} 
 \hline
Algorithm & Avg. Maximal Reward\\ [0.5ex] 
 \hline\hline
\texttt{RANDPOL} & \textbf{13.894}  \\ 
 \hline
DDPG & 12.432 \\
 \hline
PPO & 13.683   \\
 \hline
$\texttt{RANDPOL}_{N}$ & 11.581  \\[1ex] 
 \hline
  \end{tabular}
  \label{tab:score}
% \end{table}
\end{wraptable}
In both the algorithms, the policy and the value function is represented by a fully connected deep neural network. DDPG uses deterministic policy while PPO uses stochastic policy, Gaussian in particular. 
We also have a random policy, where the actions are randomly sampled from the action space uniformly.
For \texttt{RANDPOL}, we used randomized networks with two hidden layers, where we tune only the top layer and weights for the bottom layers are chosen randomly at uniform with zero mean and standard deviation inversely proportional to the number of units. These random connections are not trained in the subsequent iterations. Fig. \ref{fig:minitaur} shows the learning curve for DDPG, PPO, \texttt{RANDPOL} and randomly sampled actions. We found that \texttt{RANDPOL} gives better performance compared to both DDPG and PPO. We also tried a variation of \texttt{RANDPOL} where we fix the weights with normal distribution, which we call $\texttt{RANDPOL}_{N}$. Table \ref{tab:score} shows the average maximal score of different algorithm. We found that $\texttt{RANDPOL}_{N}$ sometimes chooses a larger weight for a connection which can degrade the performance compared to uniformly sampled weights.

\begin{figure}
	\includegraphics[width=2.7in, height =2in]{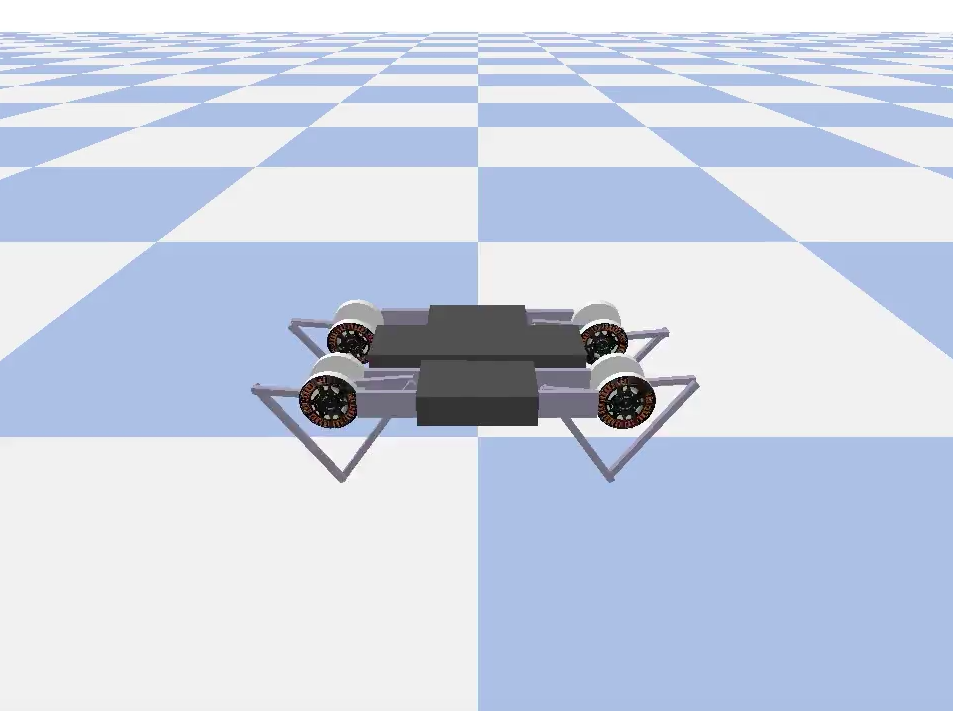}
	\includegraphics[width=2.7in, height =2in]{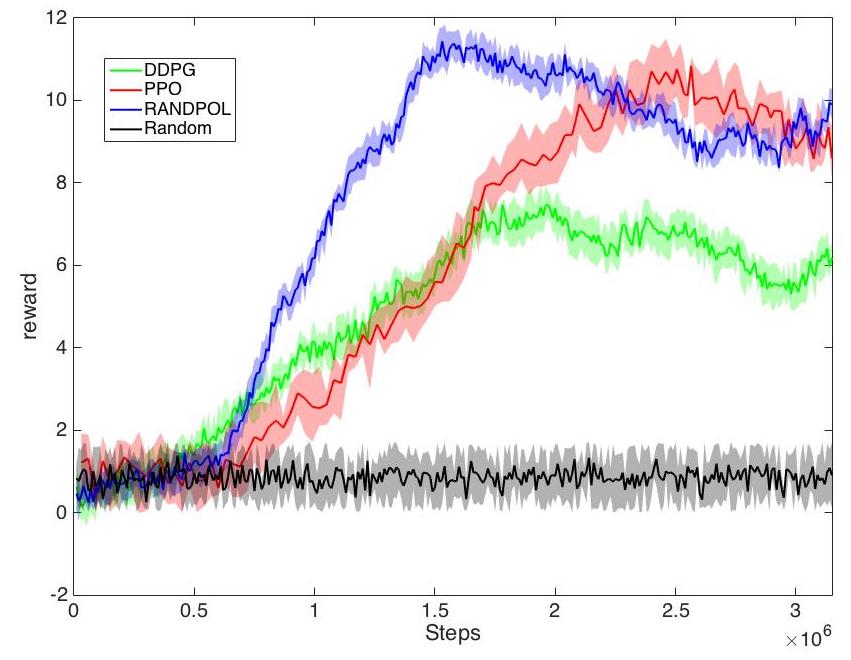}
	\caption{\textbf{Left:} Minitaur environment in PyBullet. \textbf{Right:} Avg. reward for \texttt{RANDPOL}, PPO and DDPG, averaged over five random seeds. Code and video of trained agent available on \texttt{https://github.com/RANDPOL}}
\label{fig:minitaur}
\end{figure}

\section{Conclusions}
We presented \texttt{RANDPOL}, an actor-critic algorithm where both the value and policy are represented using function approximation with random basis. In such networks, the bottom layers are randomly clamped (and thus fixed for the rest of the training) and only the last layer is fine-tuned. This reduces the number of parameters which need training and in fact, improves the performance compared to fully connected networks. We showed that this idea of \emph{randomization is cheaper than optimization} is very effective in high-dimensional challenges like quadrupedal robot. We also analyzed the algorithm theoretically, providing non-asymptotic performance guarantees including prescriptions on
required sample complexity for specified performance bounds.

\newpage
\paragraph*{Broader Impact:} The research reported in this paper has the potential to solve the continuous reinforcement learning problems that have been difficult to solve. The entire NeurIPS research community will benefit from it. We do not anticipate anyone or anything to be put at a disadvantage due to this research. Consequences of failure will be ineffective learning, and no improvement over current methods. The method does not leverage biases in the data.
\bibliographystyle{plain}
\bibliography{References}

\newpage
\appendix
Let us now bound the error in policy evaluation.
\begin{proof}[Proof of Lemma \ref{lem:one_step_poleval}]
	To begin, let $\epsilon'>0$ be arbitrary and choose $f^{*}\in\mathcal{F}\left(\Theta\right)$
	such that $\|f^{*}-G^{\pi}\,Q\|_{1,\,\mu}\leq\inf_{f\in\mathcal{F}\left(\Theta\right)}\|f-G^{\pi}\,Q\|_{1,\,\mu}+\epsilon'$.
	Using Pollard's inequality, we have
	\begin{align}
	\mathbb{P}&\left(\sup_{\widehat{f}\in\widehat{\mathcal{F}}\left(\theta^{1:J}\right)} \left| \| \widehat{f} - G^{\pi}\,Q\|_{1,\mu} -\| \widehat{f} - G^{\pi}\,Q\|_{1,\hat{\mu}}\right| > \epsilon/5 \right) \nonumber \\
	&\leq 8\, e\, (J_Q+1) \left(\cfrac{4 \,e \, Q_{\max}}{(\epsilon/5)}\right)^{J_Q} \, \exp	 \left(\cfrac{-N_Q\,\left({\epsilon}/5\right)^2}{128 \, Q_{\max}^2} \right)
	\label{eq:pollard}
	\end{align}
	where the last inequality uses the fact that the psuedo-dimension for the function class $\widehat{\mathcal{F}}\left(\theta^{1:J_Q}\right)$ is $J_Q$. 
	Now, for a given sample $(x_i,u_i)$, we have
	\[
	|G^{\pi}\, Q(x_i,u_i) -\widehat{G}^{\pi}_M\, Q(x_i,u_i) | = \gamma \bigg{|} \mathbb{E}_{x'} \left[Q(x',\pi(x') \right] - \cfrac{1}{M}\sum_{j=1}^M Q(x'_j,\pi(x_j')) \bigg{|}
	\]
	where $x_j \sim \mathbb{P}(\cdot|x_i,u_i)$ and are i.i.d. Using Hoeffding's concentration inequality followed by union bound, we have
	\begin{align*}
 \mathbb{P}\left(\max_{i=1,2,\ldots N} \bigg{|} G^{\pi}\, Q(x_i,u_i)  -\widehat{G}^{\pi}_M\, Q(x_i,u_i) \bigg{|} > \epsilon/5 \right) \nonumber 
	\leq \gamma \, N \exp\left( \cfrac{-M \, (\epsilon/5)^2}{2 Q_{\max}^2}\right)
	\end{align*}
	Hence, we have
	\begin{equation}
	\mathbb{P}\left(\|G^{\pi}\,Q   -\widehat{G}^{\pi}_M\, Q\|_{1,\hat{\mu}} > \epsilon/5 \right) < \gamma \, N \exp\left( \cfrac{-M \, (\epsilon/5)^2}{2 Q_{\max}^2}\right).
	\label{ineq:hoeffding}
	\end{equation}
	Then, choose $\hat{f}\in\widehat{\mathcal{F}}\left(\theta^{1:J_Q}\right)$
	such that $\|\hat{f}-G^{\pi}\,Q\|_{1,\,\mu}\leq\|f^{*}-G^{\pi}\,Q\|_{1,\,\mu}+\epsilon/5$
	with probability at least $1-\delta/5$ by choosing $J_Q\geq1$
	to satisfy
	\[
	\frac{C}{\sqrt{J_Q}}\left(1+\sqrt{2\,\log\frac{1}{\left(\delta/5\right)}}\right)\leq\frac{\epsilon}{5}
	\Rightarrow 
	J_Q\geq\left[\left(\frac{5 \,C}{\epsilon}\right)\left(1+\sqrt{2\,\log\frac{5}{\delta}}\right)\right]^{2}
	\]
	by Lemma \cite[Lemma 1]{rahimi2009weighted} and that $\|\widehat{f}-f^* \|_{1,\mu} \leq \|\widehat{f}-f^* \|_{2,\mu} $ by Jenson's inequality.

	Now we have the following string of inequalities, each of which hold with probability at least $1-\delta/5$:
	\begin{align}
	\| Q' - {G}^{\pi}\,Q\|_{1,\mu} &\leq \| Q' - {G}^{\pi}\,Q\|_{1,\hat{\mu}} + \epsilon/5 \label{ineq:1}\\ 
	&\leq \| Q' - \widehat{G}_{M}^{\pi}\, Q\|_{1,\hat{\mu}} + 2\epsilon/5\label{ineq:3}\\
	&\leq \| \hat{f} - \widehat{G}_{M}^{\pi}\, Q\|_{1,\hat{\mu}} + 2\epsilon/5\label{ineq:4}\\
	&\leq \|f^* - \widehat{G}_{M}^{\pi}\, Q\|_{1,\hat{\mu}} + 3\epsilon/5\label{ineq:5}\\
	&\leq \| f^* - {G}^{\pi}\,Q\|_{1,\hat{\mu}} + 4\epsilon/5 \label{ineq:7}\\
	&\leq \| f^* - {G}^{\pi}\,Q\|_{1,{\mu}} + \epsilon \label{ineq:8}\\
	&\leq \inf_{f\in\mathcal{F}\left(\Theta\right)}\|f-{G}^{\pi}\,Q\|_{1,\,\mu}+\epsilon' + \epsilon. \label{ineq:9}
\end{align}
We choose $N_Q$ from inequality (\ref{eq:pollard}) such that inequalities (\ref{ineq:1}) and (\ref{ineq:8})  hold with at least probability $1- \delta/5$. Inequalities (\ref{ineq:3}) and (\ref{ineq:7}) follow by bounding right side of (\ref{ineq:hoeffding}) by $ \delta/5$ and appropriately choosing $M$. Inequality (\ref{ineq:4}) follows from the fact that $\widehat{G}$ gives the least approximation error compared to any other function $\widehat{f} \in \widehat{\mathcal{F}}\left(\theta^{1:J_Q}\right)$.  Inequality (\ref{ineq:5}) follows by the choice of $J_Q$. The last inequality is by the choice of $f^*$. 

\end{proof}
Before we prove the bound on policy improvement approximation, let us first show some auxiliary results:
\begin{lemma}
Under Assumption \ref{assum:radon}, for any action-value function $Q$ and policy $\pi $, we have
\[
\|G\,Q - G^{\pi}\,Q \|_{1,\mu} \leq \gamma \, C_{\mu}\, \|H\, Q -H^{\pi}Q \|_{1,\mu}.
\]
\label{lem:radon}
\end{lemma}
\begin{proof}
For any state-action pair $(x,u) \in \mathcal{X} \times \mathcal{U}$, we have
\begin{align*}
G\,Q(x,u) - G^{\pi}\,Q(x,u) &= \gamma \int  \left(\max_{u'} Q(y,u') - Q(y,\pi(y))\right) dP(y|x,u)\\
&\leq \gamma \, C_{\mu}\, \int \left(H\,Q(y)- H^{\pi}\, Q(y) \right) d\mu .
\end{align*}
where we used Assumption \ref{assum:radon} for the last inequality. Also note that $G\,Q(x,u) - G^{\pi}\,Q(x,u)$ is non-negative.  Now,
\begin{align*}
\| G\,Q- G^{\pi}\,Q\|_{1,\mu} &= \int \bigg{|} G\,Q(x,u) - G^{\pi}\,Q(x,u) \bigg{|} d\mu\\
&\leq \gamma \, C_{\mu} \left( \int \left(H\,Q(y)- H^{\pi}\, Q(y) \right) d\mu \right) \\
%&\leq \gamma \, C_{\mu} \int  \left( H\,Q(y)- H^{\pi}\, Q(y) \right) d\mu\\
\end{align*}

\end{proof}
\begin{lemma}
Under Assumption \ref{assum:lips_mdp}, for any $Q\in\mathcal{F}\left(\Theta\right)$ and $\pi \in {\Pi}\left(\Theta\right)$, we have for all $x \in \mathcal{X}$,
\[
\big{|} H\,Q(x) - H^{\pi}\,Q (x) \big{|}  \leq L_U \, \|\widetilde{\pi}(x)-{\pi}(x) \|
\]
where $\widetilde{\pi}(x) = \arg \max_{u} Q(x,u)$ and $L_U= L_r + \gamma\, Q_{\max}L_p$.
\label{lem:lips_poImp}
\end{lemma}
\begin{proof}
Using the Lipschitz Assumption \ref{assum:lips_mdp}, we have
\begin{align*}
&|Q(x,u) - Q(x,u')| \\
&\hspace{-1mm}\leq |r(x,u) - r(x,u')| \hspace{-1mm}+ \gamma \hspace{-1mm}\int_{\mathcal{X}}\hspace{-1mm} \left|\left(P(dy|x,u)-P(dy|x,u') \right)\max_{\cdot}Q(y,\cdot)\right|\\
&\leq L_r |u-u'| + \gamma \, Q_{\max}\int_{\mathcal{X}} \left|P(dy|x,u)-P(dy|x,u') \right|\\
&\leq (L_r + \gamma \,Q_{\max}L_p)\, \|u-u'\|
\end{align*}
Now, we use the Lipschitz property of $Q$ function as follows:
\begin{align*}
\big{|} H\,Q(x) - H^{\pi}\,Q (x) \big{|} &= \big{|} Q(x,\widetilde{\pi}(x)) - Q (x,\pi(x)) \big{|}\\
&\leq (L_r + \gamma \,Q_{\max}L_p) \|\widetilde{\pi}(x)-\pi(x)\|.
\end{align*}
\end{proof}
\begin{proof}[Proof of Lemma \ref{lem:one_step_polimp} ]
Let $\epsilon'>0$ be arbitrary and choose $f^{*}\in\Pi \left(\Theta\right)$
	such that $\|H\, Q - H^{f^{*}}\, Q\|_{1,\,\mu}\leq\inf_{f\in \Pi \left(\Theta\right)}\|H\, Q - H^{f}\, Q\|_{1,\,\mu}+\epsilon'$. Similar to policy evaluation step, we have
	\begin{align}
	\mathbb{P}&\left(\sup_{\pi \in\widehat{\Pi}\left(\theta^{1:J_{\pi}}\right)} \bigg| \| H\, Q - H^{\pi}\, Q \|_{1,\mu} -\| H\, Q - H^{\pi}\, Q \|_{1,\hat{\mu}}\bigg| > \epsilon/3 \right) \nonumber \\
	&\leq 8\, e\, (J_{\pi}+1) \left(\cfrac{4 \,e \, \mathcal{U}_{\max}}{(\epsilon/3)}\right)^{J_{\pi}} \, \exp	 \left(\cfrac{-N_{\pi}\,\left({\epsilon}/3\right)^2}{128 \, Q_{\max}^2} \right)
	\label{eq:pollard_pi}
	\end{align}
	where the last inequality follows from Pollard's inequality and the facts that the psuedo-dimension for the underlying function class is $J_{\pi}$.	Also, note that since $H\, Q (x) - H^{\pi}\, Q (x)$ is non-negative for any $x \in \mathcal{X}$ and $\pi $ maximizes the empirical mean of action-value functions, we have for any $f \in \widehat{\Pi}\left(\theta^{1:J_{\pi}}\right)$:
\begin{align}
\| H\, Q - H^{\pi}\, Q \|_{1,\hat{\mu}} = \frac{1}{N_{\pi}} \sum_{i=1}^{N_{\pi}} \, H\, Q(x_i) - \frac{1}{N_{\pi}} \sum_{i=1}^{N_{\pi}} \, H^{\pi}\, Q(x_i) \leq \| H\, Q - H^{f}\, Q \|_{1,\hat{\mu}} \label{inequ:emp}
\end{align}
	 Now we have the following string of inequalities, each of which hold with probability $1-\delta/3$:
	\begin{align}
	\| H\, Q - H^{\pi}\, Q \|_{1,\mu} &\leq \| H\, Q - H^{\pi}\, Q \|_{1,\hat{\mu}} + \epsilon/3 \label{inequ:1}\\ 
	&\leq  \| H\, Q - H^{f}\, Q \|_{1,\hat{\mu}} + \epsilon/3 \label{inequ:2}\\
	&\leq  \| H\, Q - H^{f}\, Q \|_{1,{\mu}} + 2\epsilon/3 \label{inequ:3}\\
	&\leq   \| H\, Q - H^{f^*}\, Q \|_{1,{\mu}} + \| H^{f^*}\, Q - H^{f}\, Q \|_{1,{\mu}} + 2\epsilon/3  \label{inequ:4}\\
	&\leq  \| H\, Q - H^{f^*}\, Q \|_{1,{\mu}} + L_U \, \|f^*-f \|_{1,{\mu}} + 2\epsilon/3 \label{inequ:5}\\
	&\leq \inf_{f\in \Pi \left(\Theta\right)}\|H\, Q - H^{f}\, Q\|_{1,\,\mu}+\epsilon' + \epsilon. \label{inequ:6}
\end{align}
The inequalities (\ref{inequ:1}) and (\ref{inequ:3})  by choosing $N_{\pi}$ such that (\ref{eq:pollard_pi}) is true with atleast probability $1- \delta/3$. Inequality (\ref{inequ:2}) follows from (\ref{inequ:emp}) and  (\ref{inequ:4})  is due to triangle's inequality. To prove inequality (\ref{inequ:5}), we first use
 Lemma \ref{lem:lips_poImp} for policies $f^*$ and $f$ and then  \cite[Lemma 1]{rahimi2009weighted} such that the following holds with probability at least $1-\delta/3$: 
\begin{align}
\| H^{f^*}\, Q - H^{f}\, Q  \|_{1,{\mu}} &\leq \|f^* - f \|_{1,{\mu}} \leq \|f^* - f \|_{2,{\mu}} \nonumber\\
&\leq \frac{C'}{\sqrt{J_\pi}}\left(1+\sqrt{2\,\log\frac{1}{\left(\delta/3\right)}}\right).
\end{align}
Bounding right side by $\epsilon/3L_U$ gives us a bound on $J_{\pi}$.
 The last inequality is by the choice of $f^*$. Using Lemma \ref{lem:radon}	 concludes the lemma.
\end{proof}

Now before proving Lemma \ref{lem:errprop},  we will see how the error propagates through iterations.
\begin{lemma}\cite[Lemma 7] {antos2007value}
	\label{lem:Error} For any $K\geq1$, and $\epsilon>0$, suppose
	$\|\epsilon_{k}\|_{1,\,\mu}\leq\epsilon$ $\forall$  $k=0,\,1,\ldots,\,K-1$,
	then
	\begin{equation}
	\|Q^{\pi_K}-Q^{*}\|_{1,\,\mu}\leq2\left(\frac{1-\gamma^{K+1}}{(1-\gamma)^2}\right) \left[C_{\mu} \,\epsilon+\gamma^{K/2}\left(2\,Q_{\max}\right)\right].\label{eq:Error}
	\end{equation}
\end{lemma}
\paragraph{Choice of $K^*$}: Now, from (\ref{eq:Error}), we have  
\begin{equation*}
	\|Q^{\pi_K}-Q^{*}\|_{1,\,\mu}\leq2\left(\frac{1}{(1-\gamma)^2}\right) \left[C_{\mu}\, \epsilon+\gamma^{K/2}\left(2\,Q_{\max}\right)\right]
	\end{equation*}
which gives a bound on $K$ such that $\gamma^{K/2}\left(2\,v_{\max}\right) \leq C_{\mu}\,\epsilon$.
Denote

$$K^* = \left  \lceil \cfrac{\log\left( C_{\mu}\,\epsilon\right) - \log \left(2\, Q_{\max}\right)}{\log \gamma} \right\rceil.$$

\begin{proof}[Proof of Lemma \ref{lem:errprop}]
	 First, we show that $X_{k}\leq_{st}Y_{k}$
	holds for all $k\geq0$.
The stochastic dominance relation is the
	key to our analysis, since if we can show that $Y_{K}$ is ``small''
	with high probability, then $X_{K}$ must also be small and we infer
	that $\|Q^{\pi_K}-Q^{*}\|_{1,\,\mu}$ must be close to zero. By construction,
	 $X_{k}\leq_{st}Y_{k}$ for all $k\geq0$ 
	(see \cite[Lemma A.1]{haskell2016empirical} and \cite[Lemma A.2]{haskell2016empirical}).
	
	Next, we
	compute the steady state distribution of $\left\{ Y_{k}\right\} _{k\geq0}$
	and its mixing time. In particular, choose
	$K$ so that the distribution of $Y_{K}$ is close to its steady state
	distribution. Since $\left\{ Y_{k}\right\} _{k\geq0}$ is an irreducible
	Markov chain on a finite state space, its steady state distribution
	$\mu=\left\{ \mu\left(i\right)\right\} _{i=1}^{K^{*}}$ on $\mathcal{K}$
	exists. By \cite[Lemma 4.3]{haskell2016empirical}, the steady state distribution
	of $\left\{ Y_{k}\right\} _{k\geq0}$ is $\mu=\left\{ \mu\left(i\right)\right\} _{i=1}^{K^{*}}$
	given by: 
	\begin{align*}
	\mu\left(1\right)=\, & q^{K^{*}-1}\\
	\mu\left(i\right)=\, & \left(1-q\right)q^{K^{*}-i}, & \forall i=2,\ldots,K^{*}-1,\\
	\mu\left(K^{*}\right)=\, & 1-q.
	\end{align*}
	The constant $$\mu_{\min}\left(q;\,K^{*}\right)=\min\left\{ q^{K^{*}-1},\,\left(1-q\right)q^{\left(K^{*}-2\right)},\,\left(1-q\right)\right\} $$
	 $\forall q\in\left(0,\,1\right)$ and $K^{*}\geq1$ appears
	shortly in the Markov chain mixing time bound for $\left\{ Y_{k}\right\} _{k\geq0}$.
	We note that $\left(1-q\right)q^{K^{*}-1}\leq\mu_{\min}\left(q;\,K^{*}\right)$
	is a simple lower bound for $\mu_{\min}\left(q;\,K^{*}\right)$. Let
	$Q^{k}$ be the marginal distribution of $Y_{k}$ for $k\geq0$. By
	a Markov chain mixing time argument, we have
	\begin{eqnarray*}
	t_{\text{mix}}\left(\delta'\right) &\triangleq & \min\left\{ k\geq0\text{ : }\|Q^{k}-\mu\|_{TV}\leq\delta'\right\}\nonumber\\ 
	& \leq & \log\left(\frac{1}{\delta'\mu_{\min}\left(q;\,K^{*}\right)}\right)\nonumber\\
	& \leq & \log\left(\frac{1}{\delta'\left(1-q\right)q^{K^{*}-1}}\right)
	\end{eqnarray*}
	for any $\delta'\in\left(0,\,1\right)$.

	Finally, we conclude the argument by using the previous
	part to find the probability that $Y_{K}=1$, which is an upper bound
	on the probability that $X_{K}=1$, which is an upper bound on the
	probability that $\|Q^{\pi_K}-Q^{*}\|_{1,\,\mu}$ is below our desired
	error tolerance.	For $K\geq\log\left(1/\left(\delta'\left(1-q\right)q^{K^{*}-1}\right)\right)$
	we have $|\text{Pr}\left\{ Y_{K}=1\right\} -\mu\left(1\right)|\leq2\,\delta'$.
	Since $X_{K}\leq_{st}Y_{K}$, we have $\text{Pr}\left\{ X_{K}=1\right\} \geq\text{Pr}\left\{ Y_{K}=1\right\} $
	and so $\text{Pr}\left\{ X_{K}=1\right\} \geq q^{K^{*}-1}-2\,\delta'$.
	Choose $q$ and $\delta'$ to satisfy $q^{K^{*}-1}=1/2+\delta/2$
	and $2\,\delta'=q^{K^{*}-1}-\delta=1/2-\delta/2$ to get $q^{K^{*}-1}-2\,\delta'\geq\delta$,
	and the desired result follows.
\end{proof}

\end{document}